\documentclass[12pt]{article}
\usepackage{amsmath,amssymb,amsthm}
\usepackage{graphicx}
\usepackage{hyperref}
\usepackage{tikz}
\usepackage{authblk}

\usetikzlibrary{arrows.meta, positioning}

\usepackage[dvipsnames]{xcolor}
\definecolor{RubineRed}{HTML}{ED017D}
\definecolor{ao(english)}{rgb}{0.01, 0.75, 0.24}
\definecolor{ballblue}{rgb}{0.13, 0.67, 0.8}
\definecolor{amethyst}{rgb}{0.6, 0.4, 0.8}
\definecolor{cinnabar}{rgb}{0.89, 0.26, 0.2}

\usepackage{setspace}

\newtheorem{proposition}{Proposition}

\numberwithin{equation}{section}

\title{A Structural Theory of Cognitive Representation and Problem Solving \\
{\large Contexts, Invariance, and the Knowledge Space}}
\date{\today}

\author[1,2]{Antal Jakovac}
\author[1]{András Telcs}

\affil[1]{\footnotesize Department of Computational Sciences, HUN-REN Wigner Research Centre for Physics, H-1121 Budapest, Hungary}
\affil[2]{\footnotesize Department of Statistics, Institute of Data Analytics and Information Systems, Corvinus University of Budapest, 8 F\H{o}v\'am Square, H-1093 Budapest, Hungary}

\begin{document}
\maketitle
 
\begin{abstract}
Learning and problem solving depend critically on the structure of internal representations.  While many modern data-driven artificial systems achieve strong predictive performance, their learned representations often lack explicit structure for expressing abstraction, invariance, and task-relevant regularities.  We propose a minimal structural framework in which representational operations relevant to problem solving, such as context formation, invariance recognition, representative selection, abstraction, and procedural reuse, are made explicit. The central notion is that of a \emph{context}, formalized as a partition of a subset of an underlying state space, which fixes the distinctions, granularity, and form in which a problem can be posed. Within this setting, invariance recognition and representative selection are treated as fundamental representational operations. The framework is realized as a Knowledge Space composed of two coupled graph structures: a Concept Graph that hosts constructed and refined concepts, and a Procedure Graph that encodes typed operations over representations. Together, these structures provide a minimal cognitive-representational algebra for operating on representations without assuming sophisticated inference, learning, control, perception, or motor mechanisms. Using simple illustrative examples and a finite weak-solver demonstration, we show that appropriate representational organization can simplify the form and scope of admissible regularities, even when problem solving is carried out by a fixed and limited solver. The contribution of the paper is structural rather than algorithmic: it identifies representational prerequisites for abstraction, invariance, and procedural reuse in problem solving, and states explicit success and failure conditions for the weak-solver setting.
\end{abstract}

 \section{Introduction}

Learning and problem solving depend critically on how information is represented. While many modern data-driven artificial systems achieve strong predictive performance, their learned representations often leave implicit the structures needed for abstraction, invariance, compositionality, and task-relevant regularities \cite{BengioCourvilleVincent2013,LocatelloEtAl2019,GeirhosEtAl2020,LakeEtAl2017}. As a consequence, basic cognitive operations such as equivalence, invariance, abstraction, and relational reasoning are not explicitly available at the representational level.


The present paper addresses this question by introducing a minimal structural framework centered on the notion of \emph{context}. Contexts are formalized as finite partitions of subsets of the state space. They specify which distinctions are meaningful, at what granularity a task is posed, and in which representational space admissible solutions can be expressed. In this sense, contexts define the structural conditions under which problem solving is possible.

We treat invariance and representative selection as representational operations acting on contexts. They are not introduced as learning objectives, nor as statistical effects of data or loss functions. Instead, they are structural mechanisms that suppress irrelevant distinctions and expose simpler task-relevant laws.  We add further cognitive primitives to the framework to obtain a basic cognitive algebra.

To make these ideas explicit, we introduce the \emph{Knowledge Space}  as a minimal representational scaffold composed of two coupled graph structures: a Concept Graph   and a Procedure Graph (2G). The Concept Graph serves as a host structure for concepts that are introduced, refined, and reused during problem solving, rather than as a fixed ontology. Procedures are assumed to be weak and fixed; representational structure is the primary source of problem-solving power.

structure supports simple forms of problem solving with weak solvers. The paper does not propose a full learning architecture, a control mechanism, or
a theory of causality. Instead, it gives a self-contained representational core: context formation, invariance recognition, representative selection,
procedure storage, and a finite weak-solver demonstration with explicit success and failure conditions. Thus the paper does not ask the reader to
accept the general claim that representation helps; it specifies conditions under which representational structure reduces the burden placed on a fixed
weak solver.

\section{Related Work and Positioning}

This work is related to several lines of research in cognitive science, theoretical computer science, and artificial intelligence.  We briefly position it relative to these directions, emphasizing points
of contact and distinction.

The role of representation in problem solving has long been emphasized in cognitive science.  
Classical work by Newell and Simon highlights how representational choice determines the tractability of problem solving, independent of search strategy \cite{newell1972}. Marr’s levels of analysis further distinguish representational questions from algorithmic and implementational concerns \cite{Marr1982}. The present work operates explicitly at the representational level. Its aim is to identify the structural organization of representations that is needed before particular learning algorithms or control architectures are specified. 

\subsection{Cognitive architectures, schemas, and procedural memory}  

Our proposal is not a cognitive architecture in this sense. It does not define a complete control cycle, a perception-action loop, motivational or metacognitive subsystems, timing assumptions, or quantitative behavioral predictions. Its narrower aim is to formalize a representational substrate on which tasks, contexts, concepts, and procedures can be stated. The Knowledge Space should therefore be read as a structural layer that may support architectural mechanisms, not as a replacement for ACT-R, Soar, CLARION, LIDA, or the Common Model. In this respect the closest point of contact is not the full architecture of these systems, but their distinction between transient task states, long-term conceptual or declarative structures, and reusable procedural structures.

The framework is also related to classical schema theory. Bartlett introduced schemas as organized structures that shape remembering and interpretation \cite{Bartlett1932}. Minsky's frames represented knowledge through structured descriptions with slots, defaults, and procedural attachments \cite{Minsky1975}. Schank and Abelson's scripts represented stereotyped event structures and action sequences \cite{SchankAbelson1977}. In the present framework, contexts play a related but more formal role: a context fixes the distinctions and granularity within which a task is posed. A schema is then not taken as a primitive mental object, but as a stable pattern of representational and procedural organization that can be reconstructed from contexts, features, relations, procedures, and stored solution traces.

This distinction is important for the term ``cognitive primitive.'' We do not intend graph traversal, set union, or Boolean operations, by themselves, to carry cognitive content. They are implementation-level supports. The cognitive content lies in their representational role: reading a segmented task code, forming distinctions, comparing feature coordinates, projecting away irrelevant variation, recognizing equality or inequality, retrieving a stored procedure, composing procedures, and consolidating repeated successful traces into stable concepts or procedures. Thus the primitives are cognitive only when they operate on the representational structures supplied by contexts and the Knowledge Space.

The Procedure Graph is introduced for the same reason. It is not meant to be a production system with a complete interpreter, conflict-resolution mechanism, or control cycle. Production systems represent action or inference rules in an if--then format and are used centrally in architectures such as ACT-R and Soar. By contrast, a node of the Procedure Graph represents a typed operation over representations: for example, a projection, a comparison, a conjunction of tests, a representative-selection step, or a stored solution procedure. During problem solving, successful solution traces create new procedure nodes. Direct reuse corresponds to applying a stored node. Bounded composition creates a new procedure node linked to the procedures from which it was composed. Repeatedly successful procedures may induce equivalence edges, while procedures that discard distinctions and still preserve success may induce generalization edges.

In this way the Procedure Graph is not decorative. It records which operations are available, how they are typed, when they are reusable, and how new procedures arise from successful traces. The constructive weak-solver demonstration uses this structure explicitly: the solver begins with primitive procedures, solves tasks when possible, stores successful procedures, reuses them on later tasks, and forms new procedure nodes by bounded composition or bounded generation. The Concept Graph and Procedure Graph together therefore play the role of a minimal long-term representational store, while the current task code and its active trace play the role of a transient working structure. This analogy to working and long-term memory is structural only; no claim is made that the present model is a psychologically complete memory architecture.

Formal notions of abstraction and simplification have also been studied     in logic and symbolic AI.     Research on rule systems, theory revision, and symbolic reasoning    typically assumes that laws are already available in symbolic form and  focuses on their manipulation \cite{gabbay2008}.     In contrast, our work addresses an earlier stage by defining basic cognitive-representational primitives and the structural conditions under which they form a cognitive algebra.

\paragraph{Category-theoretic and invariance-based approaches.}

The role of invariance in representation has also been studied extensively in representation learning and structured neural architectures. Work on equivariance and equivalence of image representations, group-equivariant convolutional networks, and geometric deep learning shows how known symmetries and structural regularities can be incorporated into or measured in learned representations \cite{LencVedaldi2015,CohenWelling2016,BronsteinEtAl2021}. Our use of invariance is different. We do not build an architecture that is equivariant to a prescribed transformation group. Instead, invariance is introduced as a task-relative equivalence relation over a context, and representative selection is treated as a representational operation that may simplify the form of the task for a weak solver.
    
Modern representation learning emphasizes invariance and disentanglement
as prerequisites for generalization \cite{bengio2013}.
Architectural approaches such as equivariant neural networks encode
known symmetries directly into models.
By contrast, we treat invariance as a cognitive and structural
primitive, expressed as equivalence relations over contexts rather than
as architectural constraints.

Work on invariant risk minimization and out-of-distribution
generalization studies predictors that remain stable across multiple
environments \cite{arjovsky2019}.
These approaches operate at the level of predictors and losses.
Our focus is instead on the representational substrate: we analyze how
contexts and invariances define the space in which laws can be stated,
prior to any notion of risk or optimization.

Causal modeling further illustrates the importance of representational
structure.
Structural approaches to causality distinguish between correlation and
causal relations by imposing constraints on admissible models
\cite{pearl2009}.
In the present paper, causality is not modeled explicitly.
We only introduce minimal structural primitives, such as successor and
conditional dependence, that are necessary prerequisites for later
causal reasoning.

Within the Cognitive Systems Research community, conceptual and
structural approaches to representation have emphasized geometry,
similarity, and conceptual spaces \cite{gardenfors2000}.
Our framework shares the emphasis on structure but replaces metric
similarity with context partitions, evolving concept graphs structure
and refinement, allowing abstraction to be defined independently of geometry.

\paragraph{Rough sets and context partitions.}

In summary, the present work differs from prior approaches by treating
context as the foundational unit of task formulation and representation.
Rather than proposing new learning algorithms or inference mechanisms,
we provide a minimal structural framework.
This positioning is intended to support later work on task specification, cognitive schema and  learning,
evaluation, causality, and control, while remaining self-contained at the representational level.

\paragraph{Scope and falsifiability.}

\section{Foundations}

This section introduces the basic definitions that serve as the foundational
layer for the development of the cognitive algebra.
We present a representational scheme on which elementary constructs such as
sets, axioms, and primitive manipulations can be defined.
The goal is not to impose a specific domain interpretation, but to establish
a minimal and self-consistent structural basis for subsequent developments.

\subsection{Universe, States and Time}

We adopt a deliberately physical viewpoint, centered on the states of the world and their evolution in time. As will become clear later, this choice is not restrictive: the resulting formalism is sufficiently general to support a wide range of cognitive and representational interpretations.

The ultimate arena in which all actions take place is the Universe. Time serves as a label for different spatial instances, parameterized by a real number. It is assumed to possess an additive structure, meaning that for any three spatial instances 
$a$, $b$ and $c$, the corresponding time differences satisfy $$t_{ab}+t_{bc}=t_{ac}.$$

At each spatial instance, the Universe occupies one of its possible states. The set of all such states is denoted by $\Omega$. A state is assumed to be complete, in the sense that it uniquely determines all future states of the Universe. Under this assumption, the time evolution of the Universe can be described as a deterministic, memoryless process.

We emphasize that the Universe, its states, and the precise form of its time evolution are theoretical constructs. Their role is to support the assumption that the evolution of the Universe is deterministic, given a single actual state, even though such determinism may not be directly observable or computationally accessible.

\subsection{Object classes}

We now shift to a description from the point of view of an agent -- an intelligent entity that may be a living being, in particular a human, or an artificial system. The agent observes the Universe, understood as the ultimate underlying reality. However, the vast majority of the Universe is irrelevant for the agent’s description of the world. This is partly because the agent cannot observe most of it, partly because it lacks the computational capacity to account for all microscopic details, and most importantly because many details play no role in the fulfillment of the agent’s purpose. In biological agents this purpose is typically the maintenance of the agent within its environment, whereas in artificial agents it may be defined by externally imposed goals.

This line of reasoning implies that the agent’s relevant concepts do not correspond to individual states of the Universe, but rather to collections of states. For example, the concept of a book refers not to a single microscopic configuration of the Universe, but to a large set of states in which the book exists, irrespective of irrelevant details such as the precise position of Mars.

Consequently, for an agent, the fundamental constituents of the world are not the elementary states themselves, but subsets of the state space $\Omega$. These subsets are typically very large (cf. cylinder sets in probability theory). We refer to such subsets as object classes, or simply classes.

\paragraph{(Object) class:} subsets of the Universe, collecting states that differ only in irrelevant details.

For example:
\begin{itemize}
    \item a \emph{pen} means the same object, irrespective to the exact shape, matter content, position in the space and time
    \item living being is a collection of states where the singled out object is a living being, while all other details are irrelevant
    \item in an image the concept that a given pixel is red means that all other pixels have arbitrary color
\end{itemize}

\subsection{Context}

The object classes are the simplest elements of the description of the world. Often we want to speak about finer distinctions relative to a given concept (e.g.\ dog breeds within the class of dogs). A \emph{context} formalizes such a distinction as a partition of the underlying class.

\paragraph{Context (partition).}
Let $C \subseteq \Omega$ be an object class. A \textbf{context} on $C$ is a finite partition $\mathcal{C}$ of $C$, i.e.\ a finite family of nonempty subsets of $C$ such that
\[
\bigcup_{c \in \mathcal{C}} c = C,
\qquad
c \cap c' = \emptyset \ \text{for}\ c \neq c',
\qquad
c \neq \emptyset \ \text{for all}\ c \in \mathcal{C}.
\]
The elements $c \in \mathcal{C}$ are called the \textbf{cells} (or \textbf{context cells}) of $\mathcal{C}$. The \textbf{underling class} of the context $\mathcal C$ is $\mathrm{base}(\mathcal{C}) := C$).

Note that an object class can be represented as a context with a single element set.

\subsubsection{Contexts instead of states}

States are a semantic abstraction. In practice, an agent accesses the world exclusively through \emph{contexts}, whose cells are object classes. There is no ultimate description of the world; contexts may be refined as needed. A context determines what can be distinguished and thus sets the framework of a task to be performed or a problem to be solved.

From this point on, we treat contexts (and their cells) as the primary objects,
using states only as semantic scaffolding when needed.

\subsubsection{Refinement and induced maps}

\paragraph{Refinement.}
Let $\mathcal{C}$ and $\mathcal{C}'$ be contexts on the same underlying class $C$.
We say that $\mathcal{C}'$ is a \textbf{refinement} of $\mathcal{C}$, written
\[
\mathcal{C}' \preceq \mathcal{C},
\]
if every cell of $\mathcal{C}'$ is contained in a (unique) cell of $\mathcal{C}$:
\[
\forall c' \in \mathcal{C}'\ \exists!\, c \in \mathcal{C}\ \text{such that}\ c' \subseteq c.
\]
Equivalently, $\mathcal{C}$ is a coarsening of $\mathcal{C}'$.

Example: dog breeds is a refinement of the dogs.

\paragraph{Projection between contexts.}
If $\mathcal{C}' \preceq \mathcal{C}$, define the  projection $\Pi$
\[
\Pi_{\mathcal{C}\leftarrow \mathcal{C}'} : \mathcal{C}' \to \mathcal{C},
\qquad
\Pi_{\mathcal{C}\leftarrow \mathcal{C}'}(c') := c \ \text{where}\ c' \subseteq c.
\]
This map expresses how a finer description (cells of $\mathcal{C}'$) collapses to
a coarser description (cells of $\mathcal{C}$).

\paragraph{Local refinement.}
A context $\mathcal{C}'$ is a \textbf{local refinement} of $\mathcal{C}$ if there exists
a cell $c \in \mathcal{C}$ such that $\mathrm{base}(\mathcal{C}') = c$ (i.e.\ $\mathcal{C}'$
is a partition of one cell of $\mathcal{C}$).

\paragraph{Global refinement update.}
If $\mathcal{C}'$ is a local refinement of $\mathcal{C}$ refining a cell $c \in \mathcal{C}$,
the updated (globally refined) context is
\[
\mathcal{C}^{+} := (\mathcal{C} \setminus \{c\}) \cup \mathcal{C}'.
\]
Then $\mathcal{C}^{+}$ is a context on $\mathrm{base}(\mathcal{C})$ and satisfies
$\mathcal{C}^{+} \preceq \mathcal{C}$.

\paragraph{Common refinement.}
Let $\mathcal{C}_1$ and $\mathcal{C}_2$ be contexts on the same underlying class $C$.
Their \textbf{common refinement} is the context
\[
\mathcal{C}_1 \wedge \mathcal{C}_2 :=
\{\, c_1 \cap c_2 \mid c_1 \in \mathcal{C}_1,\ c_2 \in \mathcal{C}_2,\ c_1 \cap c_2 \neq \emptyset \,\}.
\]
It satisfies $\mathcal{C}_1 \wedge \mathcal{C}_2 \preceq \mathcal{C}_1$ and $\mathcal{C}_1 \wedge \mathcal{C}_2 \preceq \mathcal{C}_2$.
Moreover, for any context $\mathcal{D}$ on $C$ such that $\mathcal{D} \preceq \mathcal{C}_1$ and $\mathcal{D} \preceq \mathcal{C}_2$, we have
$\mathcal{D} \preceq \mathcal{C}_1 \wedge \mathcal{C}_2$.

\subsection{Characterization of contexts}

We briefly describe three basic ways of representing a context.
They are not mutually exclusive.
They serve as reference points for later constructions.

\subsubsection{Direct representation}

The simplest representation treats each cell of the context as an atomic object.
This is formally complete, but infeasible in high dimensions.
For example, a black-and-white image with one million pixels induces
$2^{1{,}000{,}000}$ possible configurations, hence a direct indexing of all cells
is not practical.

Typical instances where this representation is applied, include:
\begin{itemize}
    \item image classification, where a context cell corresponds to a class label;
    \item one-hot encoding of tokens, where each token corresponds to a distinct cell.
\end{itemize}

\subsubsection{Global coordinates}

A fine context can often be represented through several coarser contexts whose
common refinement reproduces it.

\paragraph{Coordination.}
A family of contexts $\mathcal{C}_1,\dots,\mathcal{C}_n$ is a \emph{coordination}
of a context $\mathcal{C}$ if
\[
\mathcal{C}_1 \wedge \cdots \wedge \mathcal{C}_n \;=\; \mathcal{C}.
\]

This approach is common in mathematics and in coordinate-based representations.
For example:
\begin{itemize}
    \item A $d$-dimensional real vector space can be coordinated by families of
    parallel hyperplanes. Their common intersections identify points.
    \item Embedding spaces in language models may be seen as global-coordinate
    representations, where different directions correspond to different
    contextual distinctions.
\end{itemize}

\subsubsection{Local coordinates}

A context can also be built by successive local refinements.
This yields a refinement tree and, consequently, a prefix description of cells.

\paragraph{Refinement tree.}
A \emph{refinement tree} for a context $\mathcal{C}$ is a rooted tree whose nodes
are labeled by subsets of $\mathcal{C}$ such that:
\begin{itemize}
  \item the root is labeled by $\mathcal{C}$;
  \item the children of any node form a partition of the parent label;
  \item the leaves are singletons $\{c\}$ with $c \in \mathcal{C}$.
\end{itemize}
The path from the root to a leaf defines a prefix code for the corresponding cell.

The refinement tree is not unique, even for a fixed context.
If leaf probabilities are known, one may choose a tree that minimizes expected
code length (e.g.\ Huffman coding).

Examples include:
\begin{itemize}
    \item taxonomies (e.g.\ biological classification);
    \item attribute-based product categories (e.g.\ display size, fat content);
    \item the ``twenty questions'' game.
\end{itemize}

To represent local coordinates one typically needs not only a value but also the
tested property. This motivates triplet-style descriptions (name:property:value).

\subsection{Features}

A \emph{feature} is a function defined on the cells of a context.

\paragraph{Feature.}
Let $\mathcal{C}$ be a context on $C \subseteq \Omega$.
A \emph{feature} is a map that is constant on the cells of $\mathcal C$
\[
f:C\to B,\quad f(x)=f(y)\quad \forall x,y,c: \; x,y\in c\in\mathcal C
\]
where $B$ is a value set (e.g.\ numbers or symbols).

A feature can be considered as a function whose domain is the context $f:\mathcal C\to B$, if it does not cause misunderstanding.

In measure and probability theory, a feature is called a measurable function with respect to a measurable space $(C,\sigma(\mathcal C))$, where $\sigma(\mathcal C)$ is the canonical sigma-algebra $\sigma$-algebra generated by the context $\mathcal C$.

\paragraph{Induced context.}
A feature induces a context (partition) of $C$ by grouping states according to
the feature value:
\[
f \mapsto \{\, f^{-1}(y) \mid y \in \operatorname{Range}(f),\ f^{-1}(y)\neq \emptyset \,\}.
\]

\paragraph{Extension across refinements.}
If $\mathcal{C}' \preceq \mathcal{C}$, then a feature $f:\mathcal{C}\to B$ extends
to $\mathcal{C}'$ by constancy on sub-cells: if $c' \subseteq c$ then $f(c'):=f(c)$.

\paragraph{Relation to rough-set consistency.}
The earlier notion of a relevant feature can also be expressed in the language of rough set theory. If features are treated as attributes of an information system and the task output is treated as a target or decision partition, then relevance is closely related to the consistency of an attribute-induced partition with the target partition, and hence to the background of reduct theory. The present paper does not claim this consistency condition as a new formal notion. Its aim is to place such conditions inside a cognitive task framework in which representational units may be features, partitions, representatives, laws, or admissible transformations.

\section{The Knowledge Space as Two Coupled Graphs (2G)}

We model the Knowledge Space (KS) as two coupled but distinct graph-structured
layers, collectively denoted \textbf{2G}:
(i) a \emph{Concept Graph} encoding conceptual entities and their relations,
(ii) a \emph{Procedure Graph} encoding operations as typed functions linked to
their conceptual domains and ranges.
The purpose of this section is to specify a minimal structural substrate on
which basic cognitive primitives can be defined.

\subsection{Concept Graph (CG)}

\paragraph{Definition.}
A \emph{Concept Graph} is a directed, typed graph
\[
\mathrm{CG} = (V_C, E_C, \tau_C),
\]
where $V_C$ is a set of concept nodes,
$E_C \subseteq V_C \times \Sigma_C \times V_C$ is a set of labeled edges, and
$\tau_C : V_C \to \mathcal{T}_C$ assigns each node a type from a finite set
$\mathcal{T}_C$.

\paragraph{Concept nodes.}
Nodes $v \in V_C$ represent concepts, including:
\begin{itemize}
  \item objects and attributes,
  \item sets and grouping constructs,
  \item relations (binary and higher-arity predicates),
  \item derived concepts (including invariances and representative choices).
\end{itemize}

\paragraph{Contexts and concept nodes.}
Contexts are not treated as graph objects.
They provide the semantic scaffolding for introducing and refining concepts.
Concept nodes are named distinctions that may arise at different levels of
granularity.
In particular, a concept node may represent:
(i) an underlying class $C \subseteq \Omega$,
(ii) a cell of a context defined on $C$, or
(iii) a derived construct obtained from contexts by operations such as invariance
recognition and representative selection.
In this way, contexts explain how distinctions are formed, while the Concept
Graph records which distinctions are available for reuse.

\paragraph{Edge labels.}
Edge labels encode basic structural relations between concepts.
In the permanent Concept Graph, we restrict labels to a minimal vocabulary:
\emph{equivalence}, represented by unoriented edges, and \emph{generalization}, represented by oriented edges from specific to more general concepts.
During problem solving, richer semantic relations and alternative labelings may
be introduced on task-specific copies of Concept Graph fragments.
When a problem is successfully solved, the resulting task–solution structure is
stored in a separate layer of permanent memory, together with explicit mappings
to and from the permanent Concept Graph.
These mappings link solution-specific constructs to stable concepts, and may
give rise to the introduction of new permanent concept nodes.
The mechanisms governing consolidation across memory layers are beyond the scope
of the present paper.

\subsection{Procedure Graph (PG)}

\paragraph{Definition.}
A \emph{Procedure Graph} is a directed, typed graph
\[
\mathrm{PG} = (V_F, E_F, \mathrm{dom}, \mathrm{rng}, \tau_F),
\]
where $V_F$ is a set of procedure nodes,
$E_F \subseteq V_F \times \Sigma_F \times V_F$ is a set of labeled edges,
$\mathrm{dom}$ and $\mathrm{rng}$ (domain and range) map each procedure to its input and output
types, and $\tau_F$ assigns each procedure node a procedure type.

\paragraph{Procedures as typed functions.}
Each node $p \in V_F$ represents a (possibly partial) function
\[
p : \mathrm{dom}(p) \to \mathrm{rng}(p),
\]
where input and output types refer to concept-structured objects
(e.g.\ sets, relations, graphs, coordinate tuples).

\paragraph{Edge labels in the Procedure Graph.}
We distinguish two fundamental types of edges.

\emph{Equivalence.}
An unoriented edge between two procedures represents equivalence.
Two procedures
$p : D_p \to R_p$ and $q : D_q \to R_q$
are equivalent if there exist bijections
$a : D_p \to D_q$ and $b : R_p \to R_q$
such that the diagram commutes  \ref{fig:procedure-equivalence}:
\[
q \circ a = b \circ p .
\]
Equivalently, the two procedures differ only by a relabeling of inputs and
outputs and induce isomorphic partitions of their domains.
Equivalence edges identify procedures that implement the same structural
transformation up to representation.

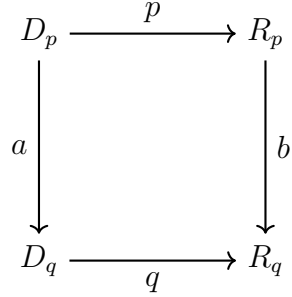
\begin{figure}[t]
\centering
\begin{tikzpicture}[
    node distance=3cm,
    auto,
    thick
]
\node (Dp) {$D_p$};
\node (Rp) [right of=Dp] {$R_p$};
\node (Dq) [below of=Dp] {$D_q$};
\node (Rq) [below of=Rp] {$R_q$};

\draw[->] (Dp) -- node[above] {$p$} (Rp);
\draw[->] (Dq) -- node[below] {$q$} (Rq);
\draw[->] (Dp) -- node[left] {$a$} (Dq);
\draw[->] (Rp) -- node[right] {$b$} (Rq);
\end{tikzpicture}
\caption{Equivalence of procedures.
Procedures $p : D_p \to R_p$ and $q : D_q \to R_q$ are equivalent if there exist
bijections $a$ and $b$ such that $q \circ a = b \circ p$.
}
\label{fig:procedure-equivalence}
\end{figure}

\emph{Generalization.}
An oriented edge $p \to q$ represents procedural generalization.
It indicates that $q$ is a more general procedure than $p$, in the sense that
$q$ collapses distinctions preserved by $p$ and induces a coarser partition of
the domain.
Thus, procedural generalization mirrors abstraction in the Concept Graph and
reflects a loss of representational detail rather than execution order or
control flow.

These two edge types constitute the complete labeling vocabulary of the
Procedure Graph in the present paper.
Procedural composition, when well-typed, is defined independently of these
edges and follows from the underlying algebra.

\paragraph{Examples.}
The Procedure Graph contains nodes representing primitive procedures operating
on structured representations. Typical examples include:
\begin{itemize}
  \item \emph{Representative selection} operators, which choose a canonical
        representative within an equivalence class induced by invariance;
  \item \emph{Abstraction} operators, such as projection or quotienting, which
        eliminate nuisance distinctions and induce coarser partitions;
  \item \emph{Derived-relation} constructors, which form new relations from
        existing ones within a fixed context;
  \item \emph{Simple inference or transformation} operators used by a solver to
        manipulate structured representations.
\end{itemize}

\paragraph{Remarks.}
In the present paper, the Procedure Graph is not used to model planning,
control, or execution order.
Its role is to organize primitive procedures by equivalence and generalization,
providing a structural account of how invariance and representational
simplification are expressed at the procedural level.

The purpose of 2G in this paper is to define a minimal structural substrate for
representations, not to propose a complete cognitive architecture.
Explicit goal representations, executive control, and long-horizon planning are
not introduced.
Instead, we use 2G to localize invariance and representative selection as
representation-level operations on structured knowledge, and we evaluate their
effect using simple solvers.

 During problem solving, task-specific copies of graph fragments are instantiated in transient memory; successful solutions are transferred to permanent memory together with mappings to the underlying conceptual structures.  The internal organization of memory is not modeled in detail here. 
\section{A Primitive Cognitive Algebra}

We introduce a minimal formal substrate together with genuinely cognitive
principles.
It is useful to separate (i) background formal assumptions required for
well-formed reasoning from (ii) cognitive principles governing structural
simplification.

\subsection{Formal assumptions}

\paragraph{Assumption A1 (Typing and compatibility).}
We assume a set of abstract type labels $\mathcal{T}$.
Types are assigned only to \emph{basic entities}, namely concept nodes and
procedure interfaces (domains and ranges).
Typing serves solely to express compatibility of entities under operations;
types do not denote sets, collections, or semantic domains.

\paragraph{Assumption A2 (Structural set formation).}
For each type label $\tau \in \mathcal{T}$, there exists a collection type
$\mathsf{Set}_\tau$.
An expression $x : \tau$ means that $x$ is compatible with type $\tau$.
Elements of $\mathsf{Set}_\tau$ are finite collections of expressions $x$ such
that $x : \tau$ holds for each element.

The empty collection $\emptyset_\tau : \mathsf{Set}_\tau$ and singleton
collections $\{x\} : \mathsf{Set}_\tau$ are admissible whenever $x : \tau$.
No semantic interpretation of $\tau$ as a universe of elements is assumed;
types function solely as structural labels governing admissibility.

\paragraph{Assumption A3 (Extensionality).}
For $A,B \in \mathsf{Set}_\tau$,
\[
A = B \iff \forall x : \tau,\; (x \in A \Leftrightarrow x \in B).
\]

\paragraph{Assumption A4 (Functions and composition).}
If $f:\tau\to\sigma$ and $g:\sigma\to\rho$, then $g\circ f:\tau\to\rho$ is
well-defined.
For each $\tau$ there exists $id_\tau:\tau\to\tau$.

\paragraph{Assumption A5 (Relations).}
Relations are predicates returning truth values on well-typed arguments.

\paragraph{Assumption A6 (Equality and compatibility).}
For each type $\tau$, equality $=_\tau$ is an equivalence relation on entities of
type $\tau$.
A procedure $f : \tau \to \sigma$ is said to be \emph{type compatible}
if
\[
x =_\tau y \;\Rightarrow\; f(x) =_\sigma f(y)
\]
whenever both sides are defined.
No assumption is made that all procedures are equality-compatible. This assumption is particularly important for representative selection (see in the next section).

\subsection{Core Cognitive Schemas and Algorithmic Support}

The cognitive algebra introduced in this paper provides a small set of
representational primitives and foundational schemas.
From these, a wide range of more complex cognitive schemas can be constructed.
In this section, we briefly revisit a few core schemas introduced in earlier
work and clarify how they are supported within the present framework.
The list is not intended to be exhaustive; rather, it identifies essential
schemas that recur across problem-solving tasks.

In earlier work, we introduced the schemas of \emph{generalization},
\emph{abstraction}, and \emph{extension} in a coordination-based setting.
In the present paper, we do not redevelop these schemas in detail.
Instead, we indicate how they arise naturally from the primitives and
foundational schemas introduced above.

Abstraction is realized by combining invariance-based equivalence with
representative selection. By selecting a representative for each equivalence class, irrelevant
distinctions are suppressed while task-relevant structure is preserved.
Abstraction therefore operates at the level of representation rather than at the
level of inference or decision-making.

Generalization concerns the construction of a joint description that covers
multiple classes (cf.\ \cite{jakovac2025}). Suppose we are given a collection of subsets
$\Omega_j \subseteq \Omega$, $j = 1,\dots,J$, each corresponding to a class of world states.
Within the present framework, each subset $\Omega_j$ induces a family of distinctions, represented by those coordinates or attributes that take fixed values across all states in $\Omega_j$.

The joint generalization is obtained by identifying the distinctions that are shared by all observed subsets.
Formally, let $S_j$ denote the set of indices of coordinates whose values are constant over $\Omega_j$, and define
\[
S^\ast = \bigcap_{j=1}^J S_j .
\]
The generalized class $\Omega^\ast$ is then defined as the set of states that satisfy all shared distinctions specified by $S^\ast$. In this construction, $S^*$ represents the joint abstraction of the observed classes, while $\Omega^*$ is the generalized class determined by the shared distinctions.

Extension corresponds to enriching the active conceptual structure when existing
distinctions are insufficient.
Within the 2G framework, this is realized by introducing new concept nodes or
relations in the Concept Graph, together with contexts that support the new
distinctions.
Extension increases representational capacity and is therefore treated
conservatively.

A particularly important higher-level operation is \emph{change of
representation}.  This refers to restructuring the representational space itself, for example by
introducing new coordinates, redefining contexts, or shifting between alternative invariant representations.

\paragraph{Algorithmic primitives.}
small collection of standard algorithmic primitives.
These primitives are not learned and are not domain-specific; they serve only to
operate on the representational structures provided by the Knowledge Space.
The assumed classes of primitives include:
\begin{itemize}
  \item \emph{Graph primitives:} traversal, reachability queries, and bounded
        search on the Concept and Procedure Graphs;
  \item \emph{Representative selection primitives:} procedures that select a
        concrete representative for equivalence classes induced by recognized
        invariances;
  \item \emph{Abstraction primitives:} projection and elimination operations that
        suppress distinctions or concept nodes that are irrelevant within a
        given context;
  \item \emph{Clustering primitives:} fixed, parametrized grouping procedures
        used to expose repeated or similar structure, without adaptive learning;
  \item \emph{Set manipulation primitives:} union, intersection, and set
        difference, applied to sets of concepts, cells, or relations;
  \item \emph{Set decision primitives:} tests such as emptiness, membership, and
        subset inclusion;
  \item \emph{Boolean primitives:} conjunction, disjunction, and negation,
        applied to the outputs of predicates or decision primitives.
\end{itemize}

The collection of primitives listed above is intentionally minimal. It suffices for the structural constructions used in the present paper and for the finite weak-solver demonstration in Appendix~\ref{app:weak-solver}. The primitives should not be read as general programming operations alone. Their cognitive role is fixed by the structures on which they act: contexts, feature coordinates, concept nodes, procedure nodes, and solution traces. In this restricted setting, they support three operations that are used explicitly below: direct application of available procedures, bounded composition of stored procedures, and bounded generation followed by testing.

\section{Problem Solving with Simple Solvers on 2G}

We describe how basic problem-solving behavior can be expressed within the 2G
framework using only simple solvers and the primitives above.
In this paper, a \emph{problem} is the task of identifying a valid and compact
rule or transformation consistent with given constraints.
We do not address long-horizon planning or planning-and-control.

The informal examples in this section are complemented by the finite
demonstration in Appendix~\ref{app:weak-solver}. There the task format,
initial Knowledge Space, solver stages, update rules, worked trace, and
failure conditions are specified explicitly. The examples below therefore
serve to motivate the representational operations, while the appendix gives
the corresponding testable weak-solver mechanism.

Across the examples, problem solving follows the same abstract pattern:
\begin{enumerate}
  \item Input is mapped into a structured representation, creating a transient
        subgraph of the Concept Graph within the Knowledge Space.
  \item Cognitive primitives and schema operate on this structure, producing
        alternative representations through invariance recognition, representative
        selection, or contextual projection.
  \item A fixed and deliberately weak solver evaluates admissible representations
        under explicit constraints and instantiates a valid solution when one
        exists.
\end{enumerate}
The solver itself does not change across tasks.
Any reduction in complexity or increase in generality is attributed to the
structure imposed by the representation and the context in which the task is
formulated.

\subsection{Example 1: Permutation invariance}

\paragraph{Task.}
Let $x=(x_1,\dots,x_n)$ be a finite sequence of symbols drawn from a finite
alphabet.
The task is to decide whether the sequence contains at least two identical
elements:
\[
f(x)=\mathbf{1}\{\exists i\neq j:\ x_i=x_j\}.
\]

Observation - The order of elements is irrelevant for this task.
A raw representation treats different permutations of $x$ as distinct inputs,
which leads to unnecessarily complex descriptions.

Invariance recognition. - Permutation invariance induces an equivalence relation on sequences:
$x\sim y$ if $y$ is obtained from $x$ by a permutation.

Representative selection. - Once this invariance is recognized, a representative may be selected for each
equivalence class.

In this example, sorting the sequence provides a simple choice of proper
representation, but alternative selections (e.g.\ multiset encodings) would
serve the same role.

This example serves only to illustrate how representational refinement, via
invariance recognition and representative selection, changes which solutions are
expressible in a given context, without altering the solver itself.

\subsection{Example 2: Invariance under nuisance coordinates}

Let $x=(x_1,\dots,x_d)\in\{0,1\}^d$.
There exists an unknown subset $S\subset\{1,\dots,d\}$ such that
\[
f(x)=\bigoplus_{i\in S} x_i,
\]
where $\oplus$ denotes XOR.
The task is to recover or evaluate the governing law from data.

In this task, only the coordinates indexed by a subset $S$ influence the output
$f(x)$. All remaining coordinates are irrelevant to the task and introduce variation
that does not participate in any dependency between inputs and outputs.

This situation gives rise to an invariance: changing coordinates outside $S$
does not affect the value of $f(x)$. Accordingly, an equivalence relation is induced on the task context, where two inputs are equivalent if they agree on all coordinates in $S$.

Representative selection fixes a concrete encoding for each equivalence class.
A natural choice is to project inputs onto the coordinates in $S$, discarding
all others.  This operation constitutes a choice of proper representation; its cognitive
content lies in the restriction of admissible distinctions, not in the specific
projection mechanism.

Within the reduced representation, admissible solutions can be expressed as
relations over $|S|$ variables. The same solver applied to the unreduced context must account for irrelevant
variation, whereas in the refined context this variation is structurally
excluded.

As in the previous example, the task specification itself does not identify
relevant coordinates. The reduction in representational complexity follows entirely from the
application of cognitive primitives at the level of contexts, independently of
how the relevant subset $S$ is discovered. This example can also be encoded as a Boolean decision table. In that encoding, the task is closely related to the classical rough-set reduct problem: one seeks a subset of attributes that preserves the distinctions needed for the decision output. The present example is not intended to introduce a new reduct algorithm. Its role is to show the same structural issue in task-context language: nuisance coordinates create distinctions that are irrelevant for the task, and a suitable task-relative representation suppresses them.

\section{Discussion and Conclusion}

This work proposes a minimal structural framework for organizing representations
so that basic cognitive primitives become explicit and usable.
The contribution is structural rather than algorithmic.

Our first claim is that the cognitive algebra and the associated structural
principles introduced here form a sufficient basis for a broad class of
representation-level operations relevant to problem solving.
Invariance recognition and representative selection provide a minimal mechanism
for suppressing irrelevant distinctions and exposing simpler task-relevant
regularities.
Within 2G, these operations are localized as procedures acting on structured
conceptual content.


Our third claim concerns the growth of the Knowledge Space under a weak
solver. The paper does not assume a strong learning mechanism. Instead,
Appendix~\ref{app:weak-solver} gives a finite demonstration in which the
solver starts from an explicit $K_0$, applies primitive procedures, composes
stored procedures, generates bounded candidate extensions, and stores only
successful procedures and traces. This gives a minimal in-paper account of
procedural reuse and representational extension, while keeping the solver
weak and the success conditions explicit.

The examples are intentionally limited.
They are meant to illustrate structural effects.
The solver is fixed and weak throughout, emphasizing that gains arise from the
representation structure rather than from algorithmic sophistication.

\textbf{Contributions and scope.}

The main contribution of this paper is the formalization of \emph{context} as the
foundational unit of distinction underlying representation, abstraction, and
problem formulation.
Contexts are defined as partitions of subsets of the state space and provide a
uniform mechanism for expressing relevance, refinement, and task-specific
structure.

Building on this foundation, we introduce a minimal structural framework—the
Knowledge Space with two coupled graphs (2G)—that situates concepts and 
 procedures within a single representational scaffold.
In this framework, the Concept Graph serves as a \emph{host structure} for
concepts that are constructed, refined, and reused during learning and problem
solving, rather than as a fixed ontology.

abstraction, and generalization are not proposed as novel algorithms. Instead, they are placed within a unified structural setting that makes their representational roles and interactions explicit. In the present paper these operations are used in two ways. First, they organize the examples of representational simplification. Second, they provide the primitives and schema used in the finite weak-solver demonstration. The focus is therefore not on a complete architecture for learning or control, but on a testable representational mechanism by which simple problem solving becomes possible with weak solvers.

The issue of context–task mismatch becomes central for complex problem solving.
As problem complexity increases, failures often arise not from the absence of a
solution, but from posing the task in a context whose distinctions are
incompatible with the available data or with the question being asked.
In the present framework, such situations correspond to ill-posed tasks: no
admissible solution can be expressed within the current context or at its chosen
granularity.
Recognizing ill-posedness and revising contexts—by changing granularity,
reformulating the task, or introducing new distinctions—are therefore essential
higher-level operations.

For classical perspectives on problem simplification, see Polya.\cite{Polya}
In the present framework, simplification corresponds to representational
operations such as fixing coordinates, restricting attention to a context cell,
or moving to a coarser partition.

 
\section{Additional Contextual Examples}
Additional illustrative examples highlighting context choice and representational form are provided in the appendix.

\subsection{Example 3: Context choice and representational form}

Consider finite sequences of symbols drawn from a fixed alphabet. We compare two contexts defined on the same underlying class.

In the context $\mathcal C_{\text{pos}}$, sequences are distinguished by symbol identity at each position. Two sequences belong to different cells whenever they differ at any position.  In the context $\mathcal C_{\text{mult}}$, sequences are distinguished only by
symbol multiplicities. Sequences that differ by permutation belong to the same cell. Thus, $\mathcal C_{\text{mult}}$ is a coarsening of $\mathcal C_{\text{pos}}$.

The two contexts admit different forms of representation. In $\mathcal C_{\text{pos}}$, structure is expressed in terms of positions and ordering. In $\mathcal C_{\text{mult}}$, structure is expressed in terms of counts or set membership.  

This example shows that context choice fixes which distinctions are admissible,
independently of the underlying objects.

\subsection{Example 4: Local versus uniform context construction}

Consider a class of objects described by multiple attributes.  One context assigns values in a fixed coordinate system shared across all objects. Each coordinate induces a partition of the underlying class, and the common
refinement of these coordinate partitions yields the full context.

An alternative context is constructed by successive local refinements.  Starting from a coarse partition, a cell is refined only when a finer distinction is required for the current purpose. The resulting context is naturally represented by a refinement tree: internal nodes correspond to intermediate cells, and leaves correspond to the final cells of the induced partition.

Both constructions define valid contexts over the same underlying class, but  they differ in how distinctions are introduced and organized. A uniform coordinate system commits to a fixed representational interface  from the start. Local refinement introduces distinctions selectively and keeps the refinement history explicit.

This example isolates a structural choice: representing a context as a single  uniform coordinate product versus representing it as a refinement process. Even when the resulting partitions coincide, the induced organization of  distinctions (and the associated procedures that traverse or update them) need
not be the same.

\subsection{Finite weak solver and explicit demonstration.}

\section*{Declaration of Interest}

The authors declare that they have no known competing financial interests or personal relationships that could have appeared to influence the work reported in this paper.

\newpage
\appendix
\section*{Appendix}

\section{Constructive demonstration with a finite weak solver}
\label{app:weak-solver}


\subsection{Task code and feature coordinates}

The input to the solver is a segmented task code. The solver does not receive an image and does not learn segmentation. It receives objects already represented as finite tuples of feature values. (c.f. \cite{zhang2019}.)  

An object code is written as
\[
x=(x_1,\ldots,x_d),
\]
where each index $i\in\{1,\ldots,d\}$ denotes a feature coordinate and $x_i$ is the value observed at that coordinate. The solver may read and compare coordinate values, but it does not know their external semantic interpretation. For example, a coordinate may externally correspond to shape, colour, size, or position, but these names are not available to the solver. 

For a set of coordinates $I\subseteq\{1,\ldots,d\}$, let
\[
x|_I
\]
denote the projection of $x$ to the coordinates in $I$.

For example,
\begin{verbatim}
one-feature object:    (a)
two-feature object:    (a,b)
three-feature object:  (a,b,c)
\end{verbatim}

A task is written as
\[
T=(P,N,C),
\]
where $P$ is a finite set of positive examples, $N$ is a finite set of negative examples, and $C$ is a finite set of candidates. The task source may choose which feature coordinates are relevant, but this information is not given to the solver. Thus the main constraint is
\[
\text{segmentation is given, but relevance of coordinates is not given.}
\]

\subsection{Initial Knowledge Space}

The initial Knowledge Space $K_0$ contains only primitive operations. In the present demonstration these are:
\begin{itemize}
\item read a segmented task code;
\item read a feature coordinate;
\item compare two coordinate values;
\item test equality;
\item test inequality;
\item project an object to selected coordinates;
\item combine tests by conjunction;
\item enumerate bounded coordinate subsets;
\item apply a test to examples and candidates;
\item store a successful procedure in the Procedure Graph;
\item insert a corresponding class distinction in the Concept Graph when needed;
\item retrieve a stored procedure;
\item apply a stored procedure;
\item store a solution trace.
\end{itemize}

The solver has no hidden rule labels and no oracle access. It does not know names such as shape, fill, circle, class, operation, or rule. These are only external interpretations. All procedures created during the demonstration are anonymous executable tests over feature coordinates.

\subsection{Separating coordinate tests}

Let $I$ be a subset of feature coordinates. A coordinate subset $I$ is acceptable for a task $T=(P,N,C)$ if all positive examples agree on $I$, and every negative example differs from this positive pattern on at least one coordinate in $I$. That is,
\[
x|_I=x'|_I \quad \text{for all} x,x'\in P,
\]
and
\[
y|_I\neq x|_I \quad \text{for all} y\in N \text{ and} x\in P.
\]
An acceptable coordinate subset defines a classifier: an object is accepted if its projection to $I$ agrees with the positive pattern.

The solver searches only a bounded family of coordinate subsets. For the demonstration with at most three coordinates this family is
\[
\{1\},\{2\},\{3\},\{1,2\},\{1,3\},\{2,3\},\{1,2,3\}.
\]
If several acceptable subsets exist, the solver chooses the first one according to a fixed order, for example increasing cardinality and then lexicographic order. If this choice leads to ambiguity among candidates and no feedback resolves it, the task is treated as underdetermined.

\subsection{Weak-solver algorithm}

\begin{verbatim}
Input:
    segmented task T = (P,N,C)
    current Knowledge Space K_t

Stage 1:
    Try one-step solutions:
        primitives
        stored procedures

    If one succeeds:
        store the solution trace
        stop.

Stage 2:
    Try bounded compositions of known items:
        p_i AND p_j
        p_i after p_j
        p_i after p_j after p_k

    If one succeeds:
        store the composed procedure
        store the solution trace
        stop.

Stage 3:
    Generate a bounded list of admissible new items:
        new projection
        new abstraction
        new class distinction
        new procedure candidate

    Test each generated item in combination with existing items.

    If one succeeds:
        insert the generated item into the Knowledge Space
        store the new procedure
        store the solution trace
        stop.

Failure:
    If all three stages fail, the task is unsolved in K_t
    under the present bounds.
\end{verbatim}

This order separates three mechanisms. In Stage 1 the task is already expressible in $K_t$. In Stage 2 the task is solvable by recombining known items. In Stage 3 the task requires a bounded extension of $K_t$.

\begin{center}
\begin{tabular}{c p{4.0cm} p{6.8cm}}
\hline
Level & Mechanism & Meaning \\
\hline
1 & direct application & the task is already expressible in $K_t$ \\
2 & bounded composition & the task is solvable by recombining known items \\
3 & bounded generation & the task requires extension of $K_t$ \\
\hline
\end{tabular}
\end{center}

The solver remains weak throughout. It does not search over all possible rules. It follows the fixed order
\[
\text{apply}
\quad \longrightarrow \quad
\text{compose}
\quad \longrightarrow \quad
\text{generate bounded candidates and test}.
\]

\subsection{Knowledge Space update}

If a stored or primitive procedure $p$ solves $T$, then only the trace is added:
\[
K_{t+1}=K_t\cup\{\operatorname{trace}(T,p)\}.
\]
If a bounded composition succeeds, the composed procedure is stored:
\[
K_{t+1}=K_t\cup\{p_{\mathrm{new}},\operatorname{trace}(T,p_{\mathrm{new}})\}.
\]
Here $p_{\mathrm{new}}$ is not a new primitive. It is a stored composition of already available procedures.

If Stage 3 succeeds by using a generated item $g_\ell$, the update has the form
\[
K_{t+1}=K_t\cup\{g_\ell,p_{\mathrm{new}},\operatorname{trace}(T,p_{\mathrm{new}})\}.
\]
Only successful generated items are inserted into the Knowledge Space. Repetition may then stabilize a solution pattern into a concept node or a reusable procedure node:
\[
\text{repeated successful traces}
\longrightarrow
\text{stable concept or procedure}.
\]

\subsection{Curriculum}

The curriculum begins with feature equality and inequality and then moves to classification. The early tasks require only direct application of primitives. Later tasks require composition and finally bounded generation.

\begin{center}
\begin{tabular}{c p{4.0cm} p{7.2cm}}
\hline
Step & Input form & What the weak solver learns \\
\hline
1 & one coordinate & equality and inequality of one feature value \\
2 & one-coordinate classification & a procedure separating positive and negative examples \\
3 & two coordinates, one relevant & projection: one coordinate matters, one coordinate is ignored \\
4 & two coordinates, both relevant & conjunction of two coordinate tests \\
5 & three coordinates, one relevant & abstraction over two irrelevant coordinates \\
6 & three coordinates, two relevant & two-coordinate classification procedure \\
7 & three coordinates, all relevant & full conjunction procedure \\
8 & reuse task & application of an already stored procedure \\
9 & combination task & bounded combination of stored procedures \\
10 & extension task & bounded generation and test of a new item \\
\hline
\end{tabular}
\end{center}

\subsection{Worked trace}

A typical trace has the following form.

\begin{verbatim}
K_0:
    read_code
    read_coordinate
    equal
    not_equal
    project
    AND
    store
    retrieve
    apply

Task T_1:
    one-coordinate classification.
    WS creates proc_001.

K_1 = K_0 + proc_001 + trace(T_1,proc_001)

Task T_2:
    two-coordinate objects, one relevant coordinate.
    WS reuses equality and creates proc_002,
    which ignores the irrelevant coordinate.

K_2 = K_1 + proc_002 + trace(T_2,proc_002)

Task T_3:
    two-coordinate objects, both coordinates relevant.
    WS creates proc_003 by conjunction.

K_3 = K_2 + proc_003 + trace(T_3,proc_003)

Task T_4:
    three-coordinate objects, two relevant coordinates.
    WS creates proc_004 by projection and conjunction.

K_4 = K_3 + proc_004 + trace(T_4,proc_004)

Task T_5:
    reuse task.
    WS applies proc_004 directly.

K_5 = K_4 + trace(T_5,proc_004)

Task T_6:
    combination task.
    WS tries bounded combinations of stored procedures
    and stores the successful combination as proc_005.

K_6 = K_5 + proc_005 + trace(T_6,proc_005)

Task T_7:
    no stored procedure or bounded composition succeeds.
    WS generates a bounded list of admissible candidates.
    One generated abstraction g_1 succeeds in combination
    with stored procedures.

K_7 = K_6 + g_1 + proc_006 + trace(T_7,proc_006)
\end{verbatim}

This trace shows the intended mechanism. Nothing hidden is given to the weak solver. No oracle provides the relevant feature coordinate. Segmentation is given, but relevance is not given. The solver uses only primitives in $K_0$, stored procedures, bounded compositions, and bounded generation. The Knowledge Space grows from $K_0$ to $K_t$ through successful solution traces.

\subsection{Success and failure conditions}

\begin{proposition}[Direct separation]
Let $T=(P,N,C)$ be a task whose objects have at most three feature coordinates. Suppose there exists an acceptable coordinate subset $I$ in the solver's bounded search family, and suppose that the corresponding classifier selects a unique correct candidate from $C$. Then the weak solver finds an acceptable subset, solves the task, stores the corresponding anonymous procedure, and stores the solution trace.
\end{proposition}

\begin{proof}
The bounded search family is finite and is enumerated by the solver. Since $I$ belongs to this family, the solver eventually tests it. By assumption, $I$ separates the positive examples from the negative examples and selects a unique correct candidate. Therefore the solver accepts the test, applies it to the candidates, obtains the correct solution, and stores the resulting procedure and trace by the update rule.
\end{proof}

\begin{proposition}[Reuse]
Suppose that a procedure $p_I$ corresponding to a coordinate subset $I$ has been stored in $K_t$. If a later task $T'=(P',N',C')$ is separated by the same coordinate subset $I$, and if $p_I$ selects a unique correct candidate from $C'$, then the weak solver solves $T'$ at Stage 1.
\end{proposition}

\begin{proof}
At Stage 1 the solver tries stored procedures from $K_t$. Since $p_I\in K_t$, it is applied to $T'$. By assumption, it separates the examples and selects a unique correct candidate. Hence the task is solved before the solver enters the compositional or generation stages.
\end{proof}

\begin{proposition}[Bounded compositional solution]
Suppose that no single primitive or stored procedure solves $T$, but there exists a bounded composition of stored procedures that separates the examples and selects a unique correct candidate. Then the weak solver solves $T$ at Stage 2 and stores the successful composition as a new procedure.
\end{proposition}

\begin{proof}
Stage 2 enumerates the permitted finite set of bounded compositions. Since a successful composition belongs to this set, it is eventually tested. By assumption, it solves the task. The solver therefore stores the composition as a new procedure and records the corresponding solution trace.
\end{proof}

\begin{proposition}[Failure under bounded search]
Let $T=(P,N,C)$ be a task. If no primitive, no stored procedure, no permitted bounded composition, and no generated candidate within the prescribed generation bound solves $T$, then the weak solver fails on $T$ in $K_t$ under the present bounds.
\end{proposition}

\begin{proof}
The solver tests exactly the alternatives allowed in Stages 1--3. By assumption none of these alternatives solves the task. Hence no successful update rule is triggered, and the solver returns failure under the present bounds.
\end{proof}

\begin{proposition}[Underdetermination]
If two admissible procedures are equally compatible with the examples but lead to different candidate choices, and no feedback or additional constraint distinguishes them, then the weak solver cannot determine a unique solution from the given task code.
\end{proposition}

\begin{proof}
The available information does not distinguish between the two admissible procedures. Since they lead to different candidate choices, any selection between them would require information not present in the task code or in $K_t$. Therefore the task is underdetermined for the weak solver.
\end{proof}

These propositions make the demonstration falsifiable. The solver is not claimed to solve arbitrary tasks. It succeeds when an admissible separating procedure, composition, or bounded generated extension exists and is tested. It fails when no such item exists within the declared bounds, or when the task is underdetermined by the given information.


\begin{thebibliography}{10}

\bibitem{BengioCourvilleVincent2013}
Yoshua Bengio, Aaron Courville, and Pascal Vincent.
\newblock Representation learning: A review and new perspectives.
\newblock {\em IEEE Transactions on Pattern Analysis and Machine Intelligence}, 35(8):1798--1828, 2013.

\bibitem{LocatelloEtAl2019}
Francesco Locatello, Stefan Bauer, Mario Lucic, Gunnar Raetsch, Sylvain Gelly, Bernhard Sch{\"o}lkopf, and Olivier Bachem.
\newblock Challenging common assumptions in the unsupervised learning of disentangled representations.
\newblock In {\em Proceedings of the 36th International Conference on Machine Learning}, volume~97 of {\em Proceedings of Machine Learning Research}, pages 4114--4124. PMLR, 2019.

\bibitem{GeirhosEtAl2020}
Robert Geirhos, J{\"o}rn-Henrik Jacobsen, Claudio Michaelis, Richard Zemel, Wieland Brendel, Matthias Bethge, and Felix~A. Wichmann.
\newblock Shortcut learning in deep neural networks.
\newblock {\em Nature Machine Intelligence}, 2(11):665--673, 2020.

\bibitem{LakeEtAl2017}
Brenden~M. Lake, Tomer~D. Ullman, Joshua~B. Tenenbaum, and Samuel~J. Gershman.
\newblock Building machines that learn and think like people.
\newblock {\em Behavioral and Brain Sciences}, 40:e253, 2017.

\bibitem{jakovac2025}
Antal Jakov{\'a}c and Andr{\'a}s Telcs.
\newblock Representation and abstraction.
\newblock {\em Mathematics}, 13(10), 2025.

\bibitem{newell1972}
Allen Newell and Herbert~A. Simon.
\newblock {\em Human Problem Solving}.
\newblock Prentice-Hall, Englewood Cliffs, NJ, 1972.

\bibitem{Marr1982}
David Marr.
\newblock {\em Vision: A Computational Investigation into the Human Representation and Processing of Visual Information}.
\newblock W. H. Freeman, San Francisco, 1982.

\bibitem{Newell1990}
Allen Newell.
\newblock {\em Unified Theories of Cognition}.
\newblock Harvard University Press, 1990.

\bibitem{AndersonLebiere1998}
John~R. Anderson and Christian Lebiere.
\newblock {\em The Atomic Components of Thought}.
\newblock Lawrence Erlbaum Associates, 1998.

\bibitem{AndersonEtAl2004}
John~R. Anderson, Daniel Bothell, Michael~D. Byrne, Scott Douglass, Christian Lebiere, and Yulin Qin.
\newblock An integrated theory of the mind.
\newblock {\em Psychological Review}, 111(4):1036--1060, 2004.

\bibitem{Laird2012}
John~E. Laird.
\newblock {\em The Soar Cognitive Architecture}.
\newblock MIT Press, 2012.

\bibitem{Sun2016}
Ron Sun.
\newblock {\em Anatomy of the Mind: Exploring Psychological Mechanisms and Processes with the Clarion Cognitive Architecture}.
\newblock Oxford University Press, 2016.

\bibitem{LairdLebiereRosenbloom2017}
John~E. Laird, Christian Lebiere, and Paul~S. Rosenbloom.
\newblock A standard model of the mind: Toward a common computational framework across artificial intelligence, cognitive science, neuroscience, and robotics.
\newblock {\em AI Magazine}, 38(4):13--26, 2017.

\bibitem{Laird2022ACTRSoar}
John~E. Laird.
\newblock An analysis and comparison of {ACT-R} and {Soar}, 2022.

\bibitem{Laird2022SoarIntro}
John~E. Laird.
\newblock Introduction to {Soar}, 2022.

\bibitem{GallFruehwirth2017}
Daniel Gall and Thom Fr{\"u}hwirth.
\newblock An operational semantics for the cognitive architecture {ACT-R} and its translation to constraint handling rules, 2017.

\bibitem{Sun2001}
Ron Sun, Edward Merrill, and Todd Peterson.
\newblock From implicit skills to explicit knowledge: A bottom-up model of skill learning.
\newblock {\em Cognitive Science}, 25(2):203--244, 2001.

\bibitem{Bartlett1932}
Frederic~C. Bartlett.
\newblock {\em Remembering: A Study in Experimental and Social Psychology}.
\newblock Cambridge University Press, 1932.

\bibitem{Minsky1975}
Marvin Minsky.
\newblock A framework for representing knowledge.
\newblock In Patrick~H. Winston, editor, {\em The Psychology of Computer Vision}, pages 211--277. McGraw-Hill, 1975.

\bibitem{SchankAbelson1977}
Roger~C. Schank and Robert~P. Abelson.
\newblock {\em Scripts, Plans, Goals and Understanding: An Inquiry into Human Knowledge Structures}.
\newblock Lawrence Erlbaum Associates, 1977.

\bibitem{gabbay2008}
Dov~M. Gabbay, Odinaldo Rodrigues, and Alessandra Russo.
\newblock Belief revision in non-classical logics.
\newblock {\em The Review of Symbolic Logic}, 1(3):267--304, 2008.

\bibitem{PhillipsWilson2010}
Steven Phillips and William~H. Wilson.
\newblock Categorial compositionality: A category theory explanation for the systematicity of human cognition.
\newblock {\em PLoS Computational Biology}, 6(7):e1000858, 2010.

\bibitem{PhillipsWilson2011}
Steven Phillips and William~H. Wilson.
\newblock Categorial compositionality ii: Universal constructions and a general theory of (quasi-)systematicity in human cognition.
\newblock {\em PLoS Computational Biology}, 7(8):e1002102, 2011.

\bibitem{LencVedaldi2015}
Karel Lenc and Andrea Vedaldi.
\newblock Understanding image representations by measuring their equivariance and equivalence.
\newblock {\em International Journal of Computer Vision}, 127:456--476, 2019.

\bibitem{CohenWelling2016}
Taco Cohen and Max Welling.
\newblock Group equivariant convolutional networks.
\newblock In {\em Proceedings of the 33rd International Conference on Machine Learning}, pages 2990--2999, 2016.

\bibitem{BronsteinEtAl2021}
Michael~M. Bronstein, Joan Bruna, Taco Cohen, and Petar Veli{\v c}kovi{\'c}.
\newblock Geometric deep learning: Grids, groups, graphs, geodesics, and gauges.
\newblock {\em arXiv preprint arXiv:2104.13478}, 2021.

\bibitem{bengio2013}
Yoshua Bengio, Aaron Courville, and Pascal Vincent.
\newblock Representation learning: A review and new perspectives.
\newblock {\em IEEE Transactions on Pattern Analysis and Machine Intelligence}, 35(8):1798--1828, 2013.

\bibitem{arjovsky2019}
Martin Arjovsky, L{\'e}on Bottou, Ishaan Gulrajani, and David Lopez-Paz.
\newblock Invariant risk minimization.
\newblock {\em arXiv preprint arXiv:1907.02893}, 2019.

\bibitem{pearl2009}
Judea Pearl.
\newblock {\em Causality: Models, Reasoning, and Inference}.
\newblock Cambridge University Press, Cambridge, 2 edition, 2009.

\bibitem{gardenfors2000}
Peter G{\"a}rdenfors.
\newblock {\em Conceptual Spaces: The Geometry of Thought}.
\newblock MIT Press, Cambridge, MA, 2000.

\bibitem{Pawlak1982}
Zdzis{\l}aw Pawlak.
\newblock Rough sets.
\newblock {\em International Journal of Computer and Information Sciences}, 11(5):341--356, 1982.

\bibitem{Pawlak1991}
Zdzis{\l}aw Pawlak.
\newblock {\em Rough Sets: Theoretical Aspects of Reasoning about Data}, volume~9 of {\em Theory and Decision Library, Series D}.
\newblock Kluwer Academic Publishers, Dordrecht, 1991.

\bibitem{SkowronRauszer1992}
Andrzej Skowron and Cecylia Rauszer.
\newblock The discernibility matrices and functions in information systems.
\newblock In Roman S{\l}owi{\'n}ski, editor, {\em Intelligent Decision Support: Handbook of Applications and Advances of the Rough Sets Theory}, volume~11 of {\em Theory and Decision Library}, pages 331--362. Kluwer Academic Publishers, Dordrecht, 1992.

\bibitem{Wille1982}
Rudolf Wille.
\newblock Restructuring lattice theory: An approach based on hierarchies of concepts.
\newblock In Ivan Rival, editor, {\em Ordered Sets}, volume~83 of {\em NATO Advanced Study Institutes Series C: Mathematical and Physical Sciences}, pages 445--470. D. Reidel Publishing Company, Dordrecht, 1982.

\bibitem{GanterWille1999}
Bernhard Ganter and Rudolf Wille.
\newblock {\em Formal Concept Analysis: Mathematical Foundations}.
\newblock Springer, Berlin, 1999.

\bibitem{Sowa1984}
John~F. Sowa.
\newblock {\em Conceptual Structures: Information Processing in Mind and Machine}.
\newblock Addison-Wesley, Reading, MA, 1984.

\bibitem{GanterKuznetsov2001}
Bernhard Ganter and Sergei~O. Kuznetsov.
\newblock Pattern structures and their projections.
\newblock In {\em Conceptual Structures: Broadening the Base}, volume 2120 of {\em Lecture Notes in Computer Science}, pages 129--142. Springer, Berlin, 2001.

\bibitem{Yao2008}
Yiyu Yao.
\newblock Granular computing: Past, present and future.
\newblock In {\em 2008 IEEE International Conference on Granular Computing}, pages 80--85, Hangzhou, China, 2008. IEEE.

\bibitem{Polya}
George P{\'o}lya.
\newblock {\em How to Solve It: A New Aspect of Mathematical Method}.
\newblock Princeton University Press, Princeton, NJ, 1945.

\bibitem{zhang2019}
Chi Zhang, Feng Gao, Baoxiong Jia, Yixin Zhu, and Song-Chun Zhu.
\newblock {RAVEN}: A dataset for relational and analogical visual reasoning.
\newblock {\em arXiv preprint arXiv:1903.02741}, 2019.

\end{thebibliography}

\end{document}